\documentclass{article}

\usepackage[numbers]{natbib}
\usepackage{algorithmic}
\usepackage{algorithm}

\usepackage[preprint]{neurips_2020}

\usepackage[utf8]{inputenc} %
\usepackage[T1]{fontenc}    %
\usepackage{url}            %
\usepackage{booktabs}       %
\usepackage{amsfonts}       %
\usepackage{nicefrac}       %
\usepackage{microtype}      %
\usepackage{enumitem}
\usepackage{natbib}
\usepackage{color}
\usepackage{xcolor}
\usepackage{capt-of}
\usepackage{wrapfig}
\usepackage{multirow}
\usepackage{subcaption}
\usepackage[colorlinks = true,
            linkcolor = blue,
            urlcolor  = blue,
            citecolor = blue,
            anchorcolor = blue]{hyperref}
\usepackage{microtype}
\usepackage{graphicx}
\usepackage{booktabs} %
\usepackage{natbib}
\usepackage{comment}

\title{Experience Replay with \\
Likelihood-free Importance Weights}

\author{%
  Samarth Sinha$^*$  %
  \\
 \ \ \ University of Toronto, Vector Institute \ \ \ \ \  \\
  \texttt{samarth.sinha@mail.utoronto.ca} \\
  \And
   Jiaming Song$^*$ %
   \\
   Stanford University \\
   \texttt{tsong@cs.stanford.edu} \\
  \AND
  Animesh Garg \\
  University of Toronto, Vector Institute, Nvidia \\
  \texttt{garg@cs.toronto.edu} \\
   \And
   Stefano Ermon \\
   Stanford University \\
   \texttt{ermon@cs.stanford.edu} \\
}

\usepackage{amsmath,amsfonts,bm,amssymb,mathtools,amsthm}
\usepackage{xcolor,color}
\usepackage{tikz}
\usetikzlibrary{positioning,angles,quotes}
\usepackage{booktabs}
\usepackage{amsmath,environ}
\usepackage{subcaption}
\usepackage{multirow,multicol}
\usepackage{thmtools,thm-restate}

\newtheorem{lemma}{Lemma}

\def\eqref#1{Eq.(\ref{#1})}

\def\1{\bm{1}}

\def\norm#1{\lVert #1 \rVert}
\newcommand*\diff{\mathop{}\!\mathrm{d}}

\def\vx{{\bm{x}}}

\DeclareMathAlphabet{\mathsfit}{\encodingdefault}{\sfdefault}{m}{sl}
\SetMathAlphabet{\mathsfit}{bold}{\encodingdefault}{\sfdefault}{bx}{n}

\def\gA{{\mathcal{A}}}
\def\gB{{\mathcal{B}}}

\def\gD{{\mathcal{D}}}
\def\gE{{\mathcal{E}}}

\def\gN{{\mathcal{N}}}

\def\gP{{\mathcal{P}}}
\def\gQ{{\mathcal{Q}}}

\def\gS{{\mathcal{S}}}

\def\gX{{\mathcal{X}}}

\newcommand{\E}{\mathbb{E}}

\newcommand{\R}{\mathbb{R}}

\DeclareMathOperator*{\argmin}{arg\,min}

\newcommand{\bb}[1]{{\mathbb{#1}}}

\begin{document}

\maketitle

\begin{abstract}
The use of past experiences to accelerate temporal difference (TD) learning of value functions, or experience replay, is a key component in  deep reinforcement learning. Prioritization or reweighting of important experiences has shown to improve performance of TD learning algorithms.
In this work, we propose to reweight experiences based on their likelihood under the stationary distribution of the current policy. Using the corresponding reweighted TD objective, we implicitly encourage small approximation errors on the value function over frequently encountered states. We use a likelihood-free density ratio estimator over the replay buffer to assign the prioritization weights. 
We apply the proposed approach empirically on two competitive methods, Soft Actor Critic (SAC) and Twin Delayed Deep Deterministic policy gradient (TD3) -- over a suite of OpenAI gym tasks and achieve superior sample complexity compared to other baseline approaches. 
\end{abstract}

\renewcommand{\thefootnote}{\fnsymbol{footnote}}
\footnotetext[1]{The authors contributed equally to this work. Order determined by \url{random.org}.}
\section{Introduction}

Deep reinforcement learning methods have achieved much success in a wide variety of domains~\cite{mnih2016asynchronous,lillicrap2015continuous,horgan2018distributed}. While on-policy methods~\cite{schulman2017proximal} are effective, using off-policy data often yields better sample efficiency~\cite{haarnoja2018soft,fujimoto2018addressing}, which is critical when querying the environment is expensive and experiences are difficult to obtain. Experience replay~\cite{lin1992self} is a popular paradigm in off-policy reinforcement learning, where experiences stored in a replay memory can be reused to perform additional updates. When applied to temporal difference (TD) learning of the $Q$-value function~\cite{mnih2015human}, the use of replay buffers %
avoids catastrophic forgetting of previous experiences and improves learning. %
Selecting experiences from the replay buffers using a prioritization strategy (instead of uniformly) can lead to large empirical improvements in terms of sample efficiency~\cite{hessel2017rainbow}. Existing prioritization procedures rely on heuristics, e.g. selecting experiences with high TD error more often~\cite{schaul2015prioritized}. Intuitively, this should minimize the maximum TD error incurred in $Q$-learning. However, such a heuristic could be highly sub-optimal in actor-critic methods, where the goal is to learn the $Q$-function induced by the current policy. 
In this case, it might be more beneficial to prioritize the correction of (potentially small) TD errors on frequently encountered states, instead of focusing on large TD errors on states that might be very infrequent under the current policy.

Based on this intuition, we investigate a new prioritization strategy for actor-critic methods 
based on the likelihood (i.e., the frequency) of experiences under the stationary distribution of the current policy~\cite{tsitsiklis1997an}. 
We derive our approach by analyzing the contractive properties of the Bellman operator under policy-dependent metrics between $Q$-value functions. The typical contraction argument uses the supremum norm:
the maximum difference over state-action pairs will decrease by a constant factor. While appealing, this metric does not account for the frequency of sampling the state-action pairs under the policy. 
Euclidean distances weighted by the stationary distribution address this issue, while allowing the same contraction argument for the Bellman operator. %
We proceed to show that such an argument does not hold for $\ell_2$ distances weighted by any other distribution. Intuitively, optimizing the expected TD-error under the stationary distribution addresses the TD-learning issue in actor-critic methods, as the TD errors in high-frequency states are given more weight.

To implement and combine this prioritization scheme with existing deep reinforcement learning methods, we consider importance sampling over the replay buffer. One approach to achieve this is to estimate the density ratio between the stationary policy distribution and the replay buffer. %
Inspired by recent advances in inverse reinforcement learning~\cite{fu2017learning} and off-policy policy evaluation~\cite{grover2019bias}, we use a likelihood-free method to obtain an estimate of the density ratio from a classifier trained to distinguish different types of experiences. We consider a smaller, ``fast'' replay buffer that contains near on-policy experiences, and a larger, ``slow'' replay buffer that contains additional off-policy experiences, and estimate density ratios between near off-policy and near on-policy distributions. We then use these estimated density ratios as importance weights over the $Q$-value function update objective. This encourages more updates over state-action pairs that are more likely under the stationary policy distribution of the current policy, i.e.,  closer to the fast replay buffer. %

Our approach can be readily combined with existing approaches that learn value functions from replay buffers. We consider our approach over two competitive actor-critic methods, Soft Actor-Critic (SAC,~\cite{haarnoja2018soft}) and Twin Delayed Deep Deterministic policy gradient (TD3,~\cite{fujimoto2018addressing}). We demonstrate the effectiveness of our approach over on a suite of OpenAI gym~\cite{dhariwal2017openai} tasks based on the Mujoco simulator~\cite{todorov2012mujoco}. 
Our method outperforms other approaches to weighted experience replay~\cite{schaul2015prioritized,wang2019boosting} in terms of sample complexity over 4 out of 5 tasks with SAC and all tasks with TD3.

\section{Preliminaries}
The reinforcement learning problem can be described as finding a policy for a Markov decision process (MDP) defined as the following tuple $(\gS, \gA, P, r, \gamma, p_0)$, where $\gS$ is the state space, $\gA$ is the action space, $P: \gS \times \gA \to \gP(\gS)$ is the transition kernel, $r: \gS \times \gA \to \bb{R}$ is the reward function, $\gamma \in [0, 1)$ is the discount factor and $p_0 \in \gP(\gS)$ is the initial state distribution. The goal is to learn a stationary policy $\pi: \gS \to \gP(\gA)$ that selects actions in $\gA$ for each state $s \in \gS$, such that the policy maximizes the expected sum of rewards: $J(\pi) := \E_\pi\left[\sum_{t=0}^\infty \gamma^t r(s_t,a_t)\right]$, where the expectation is over trajectories sampled from $s_0 \sim p_0$, $a_t \sim \pi(\cdot|s_t)$, and $s_{t+1}\sim P(\cdot|s_t, a_t)$ for $t\geq 0$. 

For a fixed policy, the MDP becomes a Markov chain, so we define the state-action distribution at timestep $t$: $d^\pi_t(s, a)$, and the the corresponding (unnormalized) stationary distribution over states and actions $d_\pi(s, a) =  \sum_{t = 0}^{\infty} \gamma^t d^\pi_t(s, a)$ (we assume this always exists for the policies we consider). We can then write $J(\pi) =  \bb{E}_{d^\pi}[r(s, a)]$. For any stationary policy $\pi$, we define its corresponding state-action value function as $Q^\pi(s, a) := \E_\pi[\sum_{t=0}^\infty \gamma^t r(s_t,a_t) | s_0 = s, a_0 = a]$, its corresponding value function as $V^\pi(s) := \E_{a \sim \pi(\cdot | s)}[Q^\pi(s, a)]$ and the advantage function $A^\pi(s, a) = Q^\pi(s, a) - V^\pi(s)$. In actor-critic methods~\cite{konda2000actor}, a stochastic policy parameterized by $\phi$ can be updated via the following:
\begin{align}
    \nabla_\phi J(\pi_\phi) = \bb{E}_{d^\pi}[\nabla_\phi \log \pi_\phi(a | s) \cdot Q^\pi(s, a)]
\end{align}
A large variety of methods have been developed in the context of deep reinforcement learning, including~\cite{silver2014deterministic,mnih2016asynchronous,lillicrap2015continuous,haarnoja2018soft,fujimoto2018addressing}, where learning good approximations to the $Q$-function is critical to the success of any deep reinforcement learning method based on actor-critic paradigms.

The $Q$-function can be learned via temporal difference (TD) learning~\cite{sutton1988learning} based on Bellman equation~\cite{bellman1957markovian} %
$
    Q^\pi(s, a) = \gB^\pi{Q^\pi}(s, a);
$
where $\gB^\pi$ denotes the Bellman evaluation operator
\begin{align}
    \gB^\pi{Q}(s, a) := r(s, a) + \gamma \bb{E}_{s', a'}[Q(s', a')],
\end{align}
where in the expectation we sample the next step $s' \sim P(\cdot | s, a)$ and $a' \sim \pi(\cdot | s)$. 
Given some experience replay buffer $\gD$ (collected by navigating the same environment, but with unknown and potentially different policies), one could optimize the following loss for a $Q$-network:
\begin{align}
    L_Q(\theta; \gD) = \bb{E}_{(s, a) \sim \gD}\left[(Q_\theta(s, a) - \hat{\gB}^\pi Q_\theta(s, a))^2\right]
\end{align}
which fits $Q_\theta(s, a)$ to an estimate of the target value $\hat{\gB}^\pi[Q_\theta](s, a)$. Ideally, we want $\hat{\gB}^\pi = \gB^\pi$, i.e. we obtain infinite experiences. In practice, the target values can be estimated either via on-policy experiences~\cite{sutton1999between} or via off-policy experiences~\cite{precup2000eligibility,munos2016safe}. %

Ideally, one could learn $Q^\pi$ by optimizing the $L_Q(\theta; \gD)$ to zero with over-parametrized neural networks. However, instead of minimizing the loss $L_Q(\theta; \gD)$ directly, prioritization over the sampled replay buffer $\gD$ could lead to stronger performance. For example, prioritized experience replay (PER,~\cite{schaul2015prioritized}) is a heuristic that assigns higher weights to transitions with higher TD errors, and is applied successfully in deep $Q$-learning~\cite{hessel2017rainbow}.

\section{Prioritized Experience Replay based on Stationary Distributions}

Assume that $d$, the distribution the replay buffer $\gD$ is sampled from, is supported on the entire space $\gS \times \gA$, and that we have infinite samples from $\pi$ (so the Bellman target is unbiased)\footnote{We also do not take the gradient over the target, which is the more conventional approach.}. Let us define the TD-learning objective for $Q$ with prioritization weights $w: \gS \times \gA \to \R^{+}$, under the sampling distribution $d \in \gP(\gS \times \gA)$: %
\begin{gather}
  \hspace{-0.7em}  L_Q(\theta; d, w) = \bb{E}_{d}\left[w(s, a) (Q_\theta(s, a) - \gB^\pi Q_\theta(s, a))^2\right] \label{eq:lqdw}
\end{gather}
In practice, the expectation in $L_Q(\theta; d, w)$ can be estimated with Monte-Carlo methods, such as importance sampling, rejection sampling, or combinations of multiple methods (such as in PER~\cite{schaul2015prioritized}). Without loss of generality, we can treat the problem as optimizing the mean squared TD error under some \emph{priority distribution} $d^w \propto d \cdot w$, since:
\begin{gather}
    \argmin_\theta L_Q(\theta; d, w) = \argmin_\theta L_Q(\theta; d^w),
\end{gather}
so one could treat prioritized experience replay for TD learning as selecting a favorable \emph{priority distribution} $d^w$ (under which the $L_Q$ loss is computed) in order to improve some notion of  performance. 

If we assume that $Q_\theta$ is parametrized by neural networks with enough parameters~\cite{cai2019neural}, the choice of \emph{priority distribution} seems to make little difference to the solution, as the neural network can always drive the TD-error to zero. Nevertheless, prior works have empirically achieved practical improvements via prioritization, either in $Q$-learning or actor-critic methods~\cite{wang2019boosting}. 

In this paper, we propose to use as \emph{priority distribution} 
$d^w=d^\pi$, where $d^\pi$ is the stationary distribution of state-action pairs under the current policy $\pi$. This reflects the intuition that TD-errors in high-frequency state-action pairs are more problematic than in low-frequency ones, as they will negatively impact policy updates more severely. 
In the following section, we argue the importance of choosing $d^\pi$ from the perspective of maintaing desirable contractive properties of the Bellman operators under more general norms.
If we consider Euclidean norms weighted under some distribution $d^w \in \gP(\gS \times \gA)$, the usual $\gamma$-contraction argument for Bellman operators holds only for $d^w = d^\pi$, and not for other distributions.

\subsection{Policy-dependent Norms for Bellman Backup}

The convergence of Bellman 
updates 
relies on the fact that the Bellman evaluation operator $\gB^\pi$ is a $\gamma$-contraction with respect to the $\ell_\infty$ norm, i.e. $\forall Q, Q' \in \gQ$, where $\gQ = \{Q: (\gS \times \gA) \to \R\}$ is the set of all possible $Q$ functions:
\begin{align}
   \norm{\gB^\pi Q - \gB^\pi Q'}_\infty \leq \gamma \norm{Q - Q'}_\infty
\end{align}

While it is sufficient to show convergence results, the $\ell_\infty$ norm reflects a distance over two $Q$ functions under the worst possible state-action pair, and is independent of the current policy. 
If two $Q$ functions are equal everywhere except for a large difference on a single state-action pair $(\widetilde{s}, \widetilde{a})$ that is unlikely under $d^\pi$, the $\ell_\infty$ distance between the two $Q$ functions is large. 
In practice, however, this will have little effect over policy updates as it is unlikely for the current policy to sample $(\widetilde{s}, \widetilde{a})$. 

Since our goal with the TD updates is to learn $Q^\pi$, a distance metric that is related to $\pi$ is a more suitable one for comparing different $Q$ functions, reflecting the intuition that errors in frequent state-action pairs are more costly than on infrequent ones. Let us consider the following weighted $\ell_2$ distance between $Q$ functions, 
\begin{align}
    \norm{Q - Q'}^2_d := \bb{E}_{(s, a) \sim d}[(Q(s, a) - Q'(s, a))^2]
\end{align}
where $d \in \gP(\gS \times \gA)$ is a distribution over state-action pairs. This can be treated as the $\ell_2$ norm but measured over stationary distribution $d$ as opposed to the Lebesgue measure. This is closely tied to the $L_Q$ objective since
\begin{gather*}
    L_Q(\theta; d) = \norm{Q_\theta(s, a) - \gB^\pi Q_\theta(s, a)}_d^2
\end{gather*}
In the following theorem, we show that $\gB^\pi$ is only a contraction operator when under the $\norm{\cdot}_{d^\pi}$ norm; this supports the use of $d^\pi$ instead of other distributions for the $L_Q$ objective, as it reflects a more reasonable measurement of distance between $Q$-functions for policy $\pi$.
\begin{restatable}{theorem}{contraction}
\label{thm:contract-new-norm}
The Bellman operator $\gB^\pi$ is a $\gamma$-contraction with respect to the $\norm{\cdot}_{d}$ norm if and only if $d = d^\pi$ holds almost everywhere, i.e.,
\begin{align*}
    & \norm{\gB^\pi Q - \gB^\pi Q'}_d \leq \gamma \norm{Q - Q'}_d, \forall Q, Q' \in \gQ \iff d = d^\pi, \quad a.e.
\end{align*}
\end{restatable}
\begin{proof}
In Appendix~\ref{app:proof}.

\end{proof}

\subsection{TD Learning based on $d^\pi$}

Theorem~\ref{thm:contract-new-norm} highlights the importance of using $d^\pi$ in the $\norm{\cdot}_d$ norm specifically for measuring the distance between $Q$-functions: if we use any distribution other than $d^\pi$, the Bellman operator is not guaranteed to be a $\gamma$-contraction under that distance, which leads to worse convergence rates. 
The $\norm{\cdot}_{d^\pi}$ norm also captures our intuition that errors in high-frequency state-action pairs are more problematic than low-frequency ones, as they are likely to have larger effect in policy learning. For example, for the actor-critic policy gradient with $Q_\theta$:
\begin{align}
    \nabla_\phi J(\pi_\phi) = \bb{E}_{d^\pi}[\nabla_\phi \log \pi_\phi(a | s) Q_\theta(s, a)]
\end{align}
if $(Q_\theta(s, a) - Q^\pi(s, a))^2$ is large for high-frequency $(s, a)$ tuples, then the policy update is likely to be worse than the update with ground truth $Q^\pi$.
Moreover, the gradient descent update over the objective $L_Q(\theta; d^\pi)$:
\begin{align}
    \theta & \leftarrow \theta - \eta \nabla_\theta L_Q(\theta; d^\pi)  = \theta - \eta \bb{E}_{d^\pi}[(Q_\theta(s, a) - \hat{\gB}^\pi Q_\theta(s, a)) \nabla_\theta Q_\theta(s, a)] \nonumber
\end{align}
corresponds to a batch version of TD update.
This places more emphasis on TD errors for state-action pairs that occur more frequently under the current policy.

To illustrate the validity of using $d^\pi$, we consider a chain MDP example (Figure~\ref{fig:chain-mdp-figure}, but with 5 states in total), where the agent takes two actions that progress to the state on the left or on the right. The agent receives a final reward of 1 at the right-most state and rewards of 0 at other states. The policy takes the right action at each state with probability $p$, and takes the left action for with probability $1-p$. We initialize the $Q$-function from $[0, 1]$ uniformly at random and consider $p = 0.8$ and $0.2$.

We compare three approaches to prioritization with TD updates: uniform over all state-action pairs, prioritization over TD error (as done in~\cite{schaul2015prioritized}), and prioritization with $d^\pi$; we include more details in Appendix~\ref{app:exp}. We illustrate the $\norm{\cdot}_{d^\pi}^2$ distance between the learned $Q$-function and the ground truth $Q$-function in Figure~\ref{fig:chain-mdp-results}; prioritization with $d^\pi$ outperforms both uniform sampling and prioritization with TD error in terms of speed of converging to the ground truth, especially at the initial iterations. When $p = 0.8$, $d^\pi$ only takes 120 steps on average to decrease the expected error to be smaller than 1, while TD error takes 182 steps on average; this means that prioritization with $d^\pi$ is helpful when we have a limited update budget.

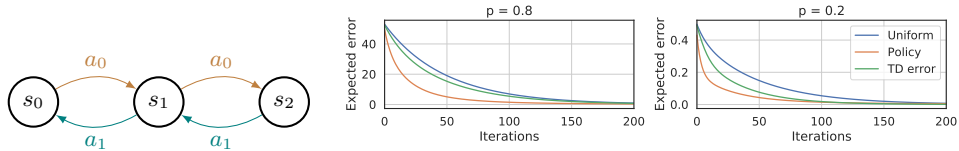
\begin{figure}
\centering
\begin{subfigure}{0.3\textwidth}
\centering
\resizebox{\textwidth}{!}{
\begin{tikzpicture}[auto,node distance=8mm,>=latex,font=\small]
    \tikzstyle{round}=[thick,draw=black,circle]

    \node[round] (s0) {$s_0$};
    \node[round,right=10mm of s0] (s1) {$s_1$};
    \node[round,right=10mm of s1] (s2) {$s_2$};

    \draw[color=brown,bend left,->] (s0) to node [auto] {$a_0$} (s1);
    \draw[color=brown,bend left,->] (s1) to node [auto] {$a_0$} (s2);
    
    \draw[color=teal,bend left,->] (s1) to node [auto] {$a_1$} (s0);
    \draw[color=teal,bend left,->] (s2) to node [auto] {$a_1$} (s1);
    
\end{tikzpicture}
}
\caption{Illustration of the chain MDP. %
}
\label{fig:chain-mdp-figure}
\end{subfigure}
~
\begin{subfigure}{0.6\textwidth}
    \centering
    \includegraphics[width=1\textwidth]{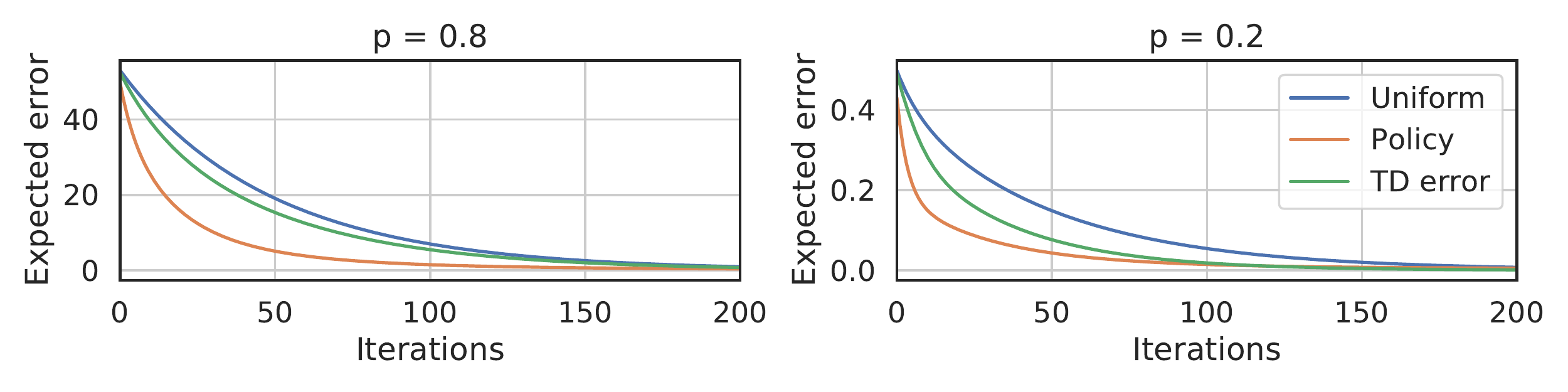}
    \caption{Estimation error ($\bb{E}_{d^\pi}[(Q_\theta - Q^\pi)^2]$) for different prioritization methods, including uniform sampling (Uniform), sampling based on TD error (TD error), and sampling based on $d^\pi$ (Policy).}
    \label{fig:chain-mdp-results}
\end{subfigure}
\caption{Simulation of TD updates with different prioritization methods.}
\end{figure}

\section{Likelihood-free Importance Weighting over Replay Buffers}
In practice, however, there are two challenges with regards to using $L_Q(\theta; d^\pi)$ as the objective. On the one hand, an accurate estimate of $d^\pi$ requires many on-policy samples from $d^\pi$ and interactions with the environment, which could increase the practical sample complexity; on the other hand, if we instead use off-policy experiences (from the replay buffer), it would be difficult to estimate the importance ratio $w(s, a) := d^\pi(s, a) / d^D(s, a)$ when the replay buffer $\gD$ is a mixture of trajectories from different policies. Therefore, likelihood-free density ratio estimation methods that rely only on samples (e.g. from the replay buffer) rather than likelihoods are more general and well-suited for estimating the objective function $L_Q(\theta; d^\pi)$.

In this paper, we consider using the variational representation of $f$-divergences~\cite{csiszar1964eine} to estimate the density ratios. For any convex, lower-semicontinuous function $f: [0, \infty) \to \R$ satisfying $f(1) = 0$, the $f$-divergence between two probabilistic measures $P, Q \in \gP(\gX)$ (where we assume $P \ll Q$, i.e. $P$ is absolutely continuous w.r.t. $Q$) is defined as:
$
    D_{f}(P \Vert Q)
     = \int_\gX f\left(\diff P(\vx) / \diff Q(\vx)\right) \diff Q(\vx) \label{eq:fdiv}
$. 
A general variational method can be used to estimate $f$-divergences given only samples from $P$ and $Q$.
\begin{lemma}[\cite{nguyen2008estimating}]
\label{thm:nwj} Assume that $f$ has first order derivatives $f'$ at $[0, +\infty)$. 
$\forall P, Q \in \gP(\gX)$ such that $P \ll Q$ and $w: \gX \to \R^{+}$,
\begin{gather}
    D_{f}(P \Vert Q) \geq \bb{E}_P[f'(w(\vx))] - \bb{E}_Q[f^{*}(f'(w(\vx)))]
\end{gather}
where $f^{*}$ denotes the convex conjugate and the equality is achieved when 
$w = \diff P / \diff Q %
$.
\end{lemma}
We can apply this approach to estimating the density ratio $w(s, a) := d^\pi(s, a) / d^D(s, a)$ with samples from the replay buffer. These ratios are then multiplied to the $Q$-function updates to perform importance weighting (we do not do re-sampling as is done in PER). Specifically, we consider sampling from two types of replay buffers. One is the \textit{regular (slow) replay buffer}, which contains a mixture of trajectories from different policies; the other is a \textit{smaller (fast) replay buffer}, which contains only a small set of trajectories from very recent policies. After each episode of environment interaction, we update both replay buffers with the new experiences; the distribution of the slow replay buffer changes more slowly due to the larger size (hence the name ``slow'').

The slow replay buffer contains off-policy samples from $d^D$ whereas the fast replay buffer contains (approximately) on-policy samples from $d^{\pi}$ (assuming the buffer size is small enough). Therefore, the slow replay buffer has better coverage of transition dynamics of the environment while being less on-policy. 
Denoting the fast and slow replay buffers as $\gD_{\mathrm{f}}$ and $\gD_{\mathrm{s}}$ respectively, we estimate the ratio $d^\pi / d^D$ via minimizing the following objective over the network $w_\psi(\vx)$ parametrized by $\psi$ (the outputs $w_\psi(s, a)$ are forced to be non-negative via activation functions):
\begin{align}
    L_w(\psi) := \bb{E}_{\gD_{\mathrm{s}}}[f^{*}(f'(w_\psi(s, a)))] - \bb{E}_{\gD_{\mathrm{f}}}[f'(w_\psi(s, a))] \label{eq:w-train-obj}
\end{align}
From Lemma~\ref{thm:nwj}, we can recover an estimate of the density ratio from the optimal $w_\psi$ by minimizing the $L_w(\psi)$ objective. To address the finite sample size issue, we apply self-normalization~\cite{cochran2007sampling} to the importance weights over the slow replay buffer $\gD_{\mathrm{s}}$ with a temperature hyperparameter $T$:
\begin{align}
    \tilde{w}_\psi(s, a) := \frac{w_\psi(s, a)^{1/T}}{\bb{E}_{\gD_{\mathrm{s}}}[w_\psi(s, a)^{1/T}]} \label{eq:w-tilde}
\end{align}
The final objective for TD learning over $Q$ is then
\begin{align}
 L_Q(\theta; d^\pi) \approx L_Q(\theta; \gD_{\mathrm{s}}, \tilde{w}_\psi)  := \bb{E}_{(s, a) \sim \gD_{\mathrm{s}}}[\tilde{w}_\psi(\vx) (Q_\theta(s, a) - \hat{\gB}^\pi Q_\theta(s, a))^2] \label{eq:final-q-obj}
\end{align}
where the target $\hat{\gB}^\pi Q_\theta$ is estimated via Monte Carlo samples. We keep the remainder of the algorithm, such as policy gradient and value network update (if available) unchanged, so this method can be adapted for different off-policy actor-critic algorithms, utilizing their respective advantages. We observe that using the weights to correct the policy updates does not demonstrate provide much marginal improvements, so we did not consider this for comparison.

We describe a general procedure of our approach in Algorithm~\ref{alg:ac-lfiw} (Appendix~\ref{app:alg}), where one can modify from some ``base'' actor-critic algorithm to implement our approach. In the experiments, we consider two popular actor-critic algorithms, Soft Actor-Critic (SAC, ~\cite{haarnoja2018soft}) and Twin Delayed Deep Deterministic policy gradient (TD3,~\cite{fujimoto2018addressing}); these algorithm cover both stochastic and deterministic policies, as our method does not require likelihood estimates from the policy.

\section{Related Work}

Experience replay~\cite{lin1992self} is a crucial component in deep reinforcement learning~\cite{hessel2017rainbow,andrychowicz2017hindsight,schaul2015prioritized}, where off-policy experiences are utilized to improve sample efficiency. These experiences can be utilized on policy updates (such as in actor-critic methods~\cite{konda2000actor,wang2016sample}), on value updates (such as in deep Q-learning~\cite{schaul2015prioritized}) or on evaluating TD update targets~\cite{precup2000eligibility,precup2001off}.

For value updates, there are two sources of randomness that could benefit from importance weights (prioritization). The first source is the evaluation of the TD learning target for longer traces such as TD($\lambda$); importance weights can be used to debias targets computed from off-policy trajectories~\cite{precup2000eligibility,munos2016safe,espeholt2018impala,schmitt2019off}, similar to its role in policy learning. The second source is the sampling of state-action pairs where the values are updated~\cite{schaul2015prioritized}, which is addressed in this paper. 

Numerous techniques have achieved superior sample complexity through prioritization of replay buffers. In model-based planning, Prioritized Sweeping~\cite{moore1993prioritized,andre1998generalized,seijen2013planning} selects the next state updates according to changes in value. 
Prioritized Experience Replay (PER,~\cite{schaul2015prioritized}) emphasizes experiences with larger TD errors and is critical to the success of sample efficient deep Q-learning~\cite{hessel2017rainbow}. Remember and Forget Experience Replay (ReF-ER,~\cite{novati2018remember}) removes the experiences if it differs much from choices of the current policy; this encourages sampling on-policy behavior which is similar to what we propose. Differing from ReF-ER, we do not require knowledge of the policy distribution. %

Convergence of TD learning under the stationary distribution of the policy has been analyzed in the context of function approximation~\cite{van1998learning,cai2019neural}. Our paper provides a practical approach to implement this in the context of deep reinforcement learning. Likelihood-free density ratio estimation  have been adopted in imitation learning (IL,~\cite{ho2016generative}), inverse reinforcement learning (IRL,~\cite{fu2017learning}) and model-based off-policy policy evaluation (OPE,~\cite{grover2019bias}). Different from these cases, we do not use the weights to estimate the advantage function or to reduce bias in reward estimation, but to improve performance of TD learning with function approximation.

\section{Experiments}
\label{sec:exp}

\begin{figure*}
\centering
\begin{subfigure}{0.27\textwidth}
\includegraphics[width=\textwidth]{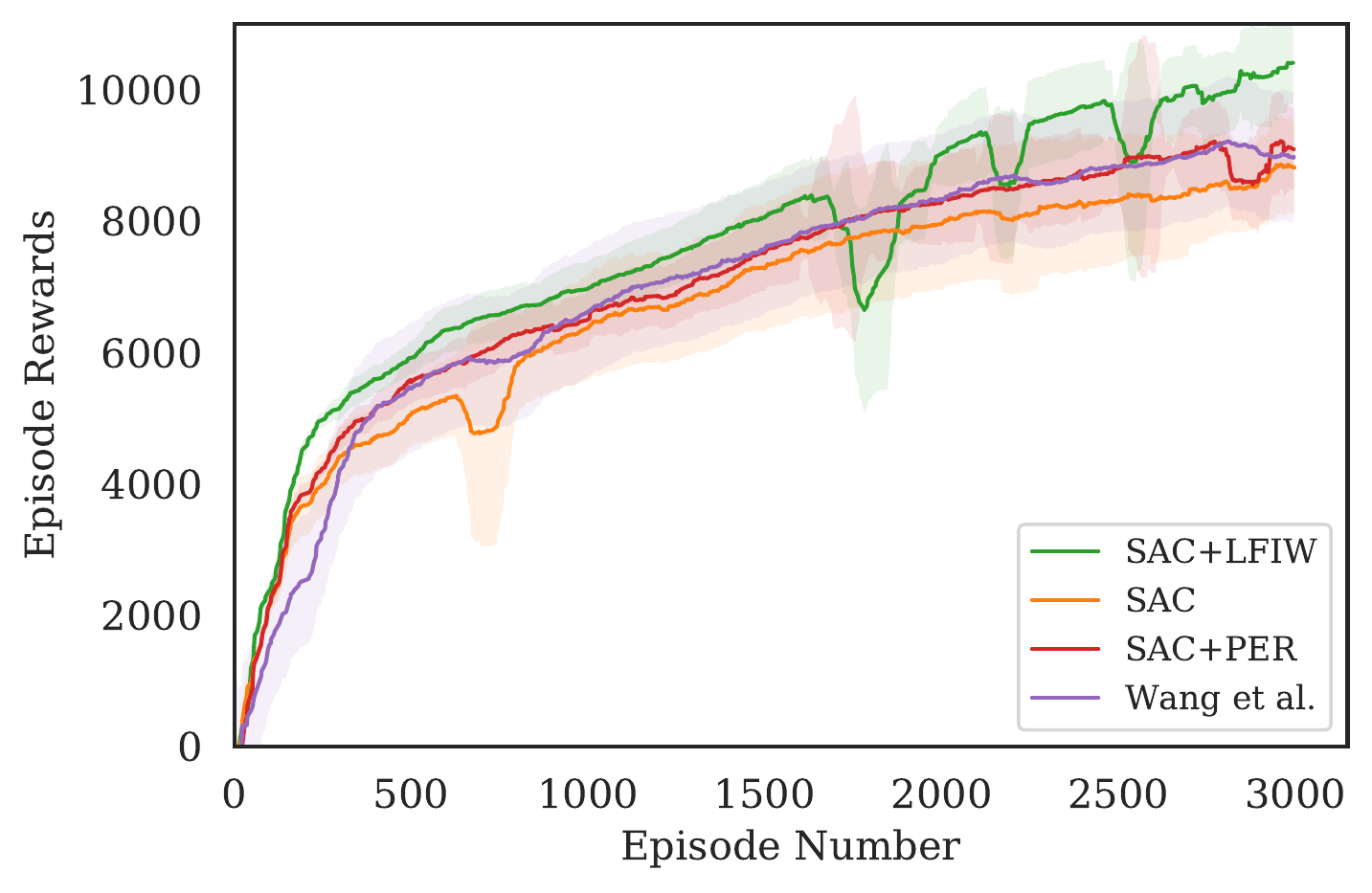}
\caption{HalfCheetah-v2}
\end{subfigure}
~
\begin{subfigure}{0.27\textwidth}
\includegraphics[width=\textwidth]{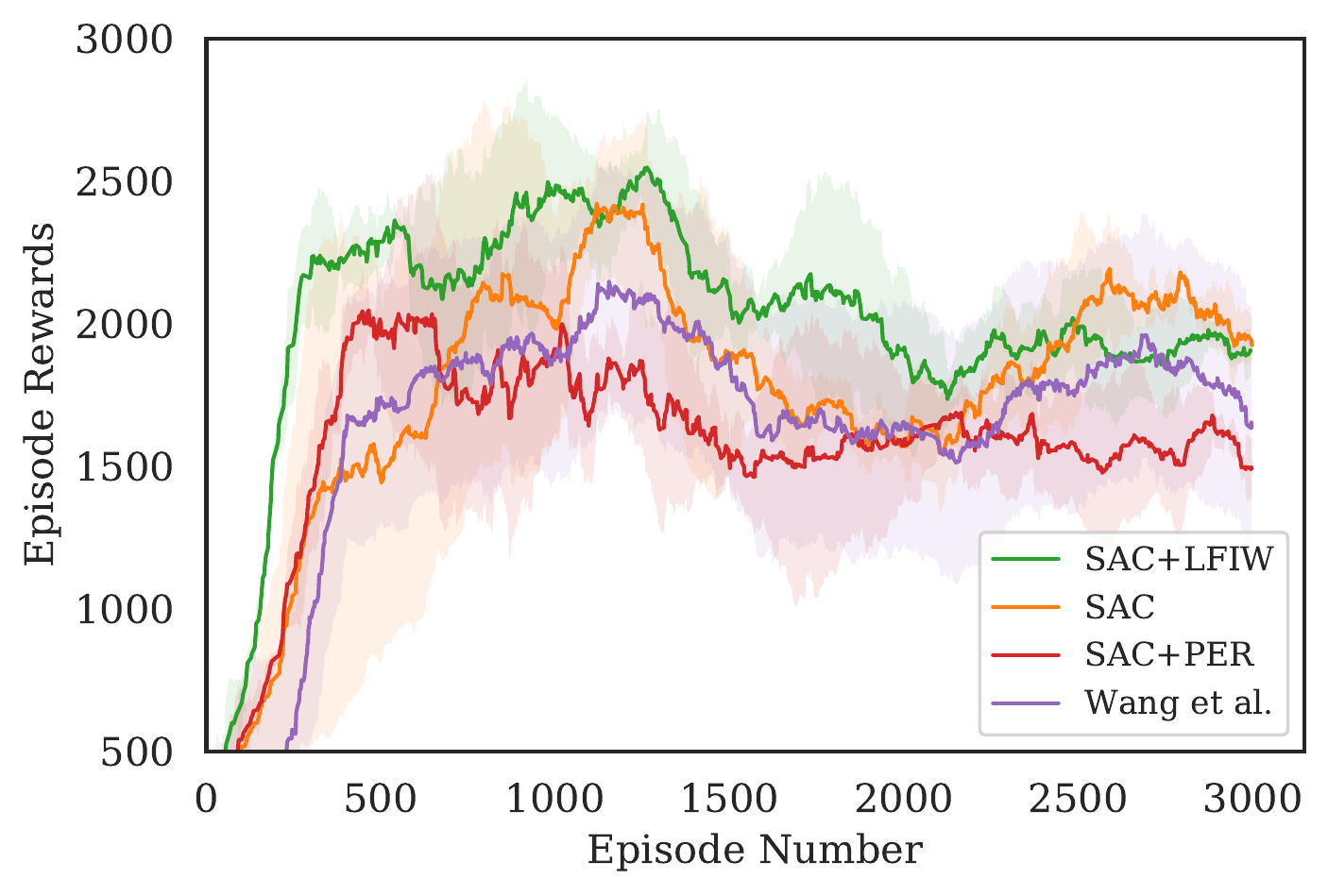}
\caption{Hopper-v2}
\end{subfigure}
~
\begin{subfigure}{0.27\textwidth}
\includegraphics[width=\textwidth]{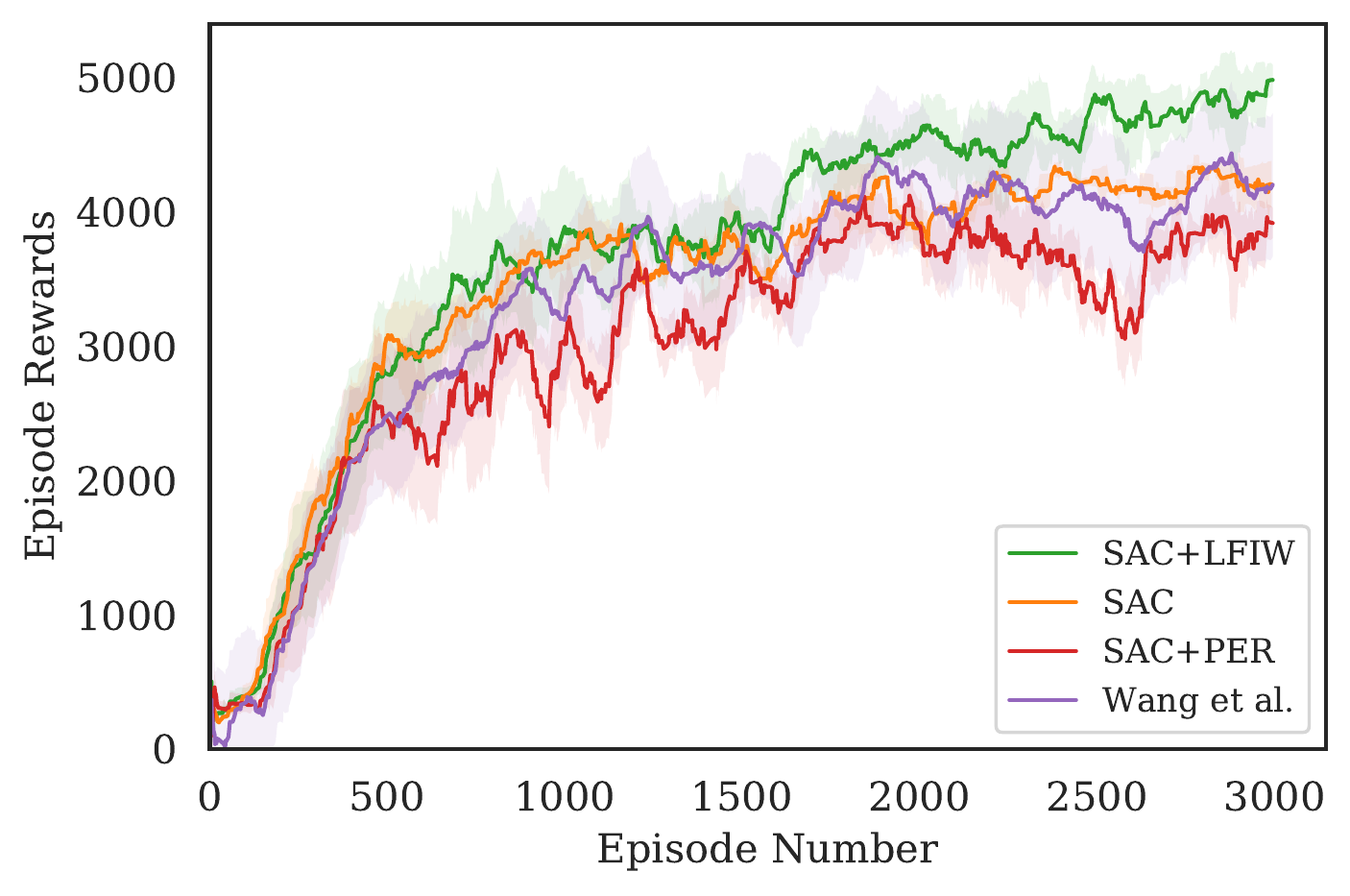}
\caption{Walker2d-v2}
\end{subfigure}
~
\begin{subfigure}{0.27\textwidth}
\includegraphics[width=\textwidth]{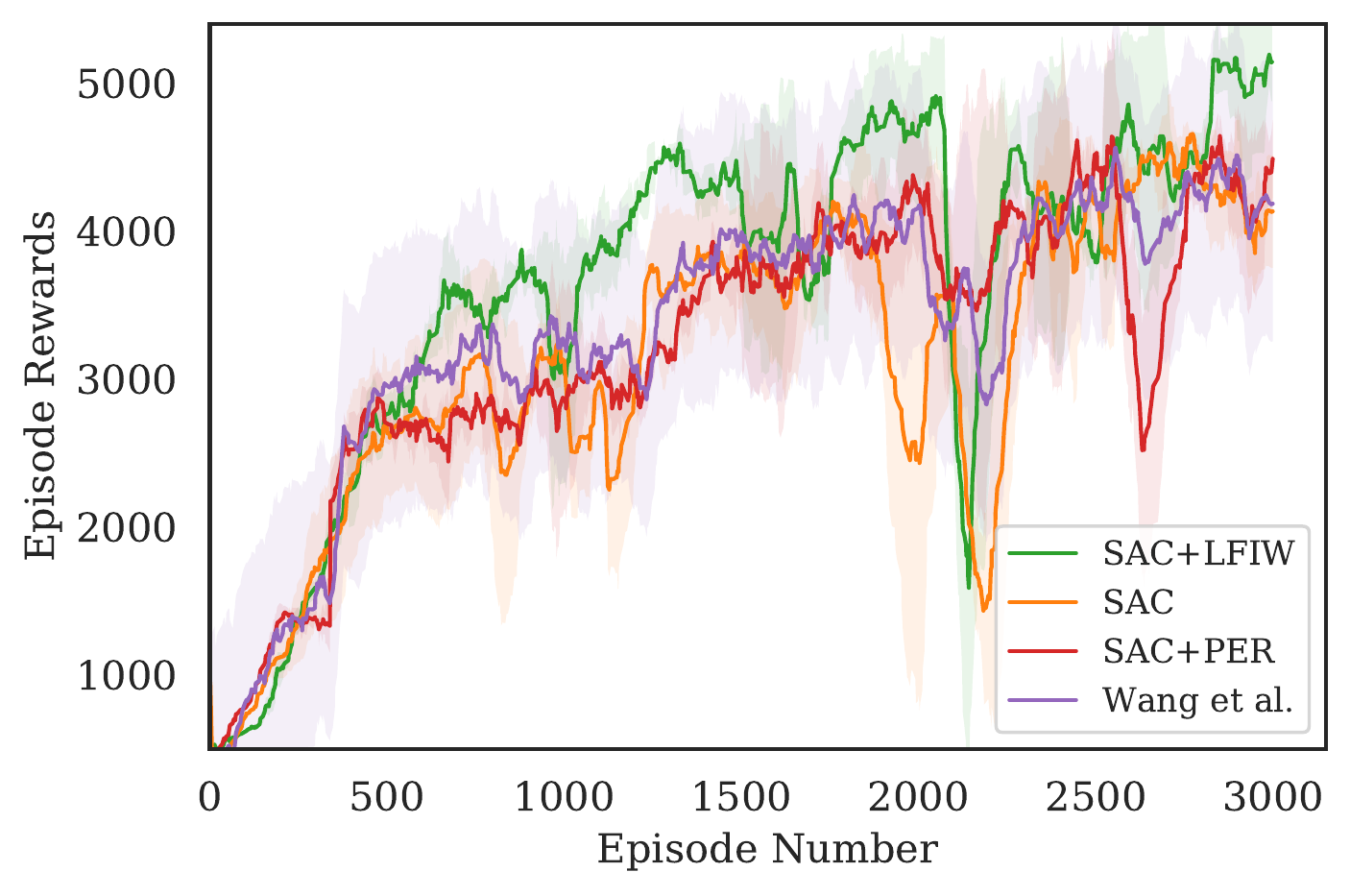}
\caption{Ant-v2}
\end{subfigure}
~
\begin{subfigure}{0.27\textwidth}
\includegraphics[width=\textwidth]{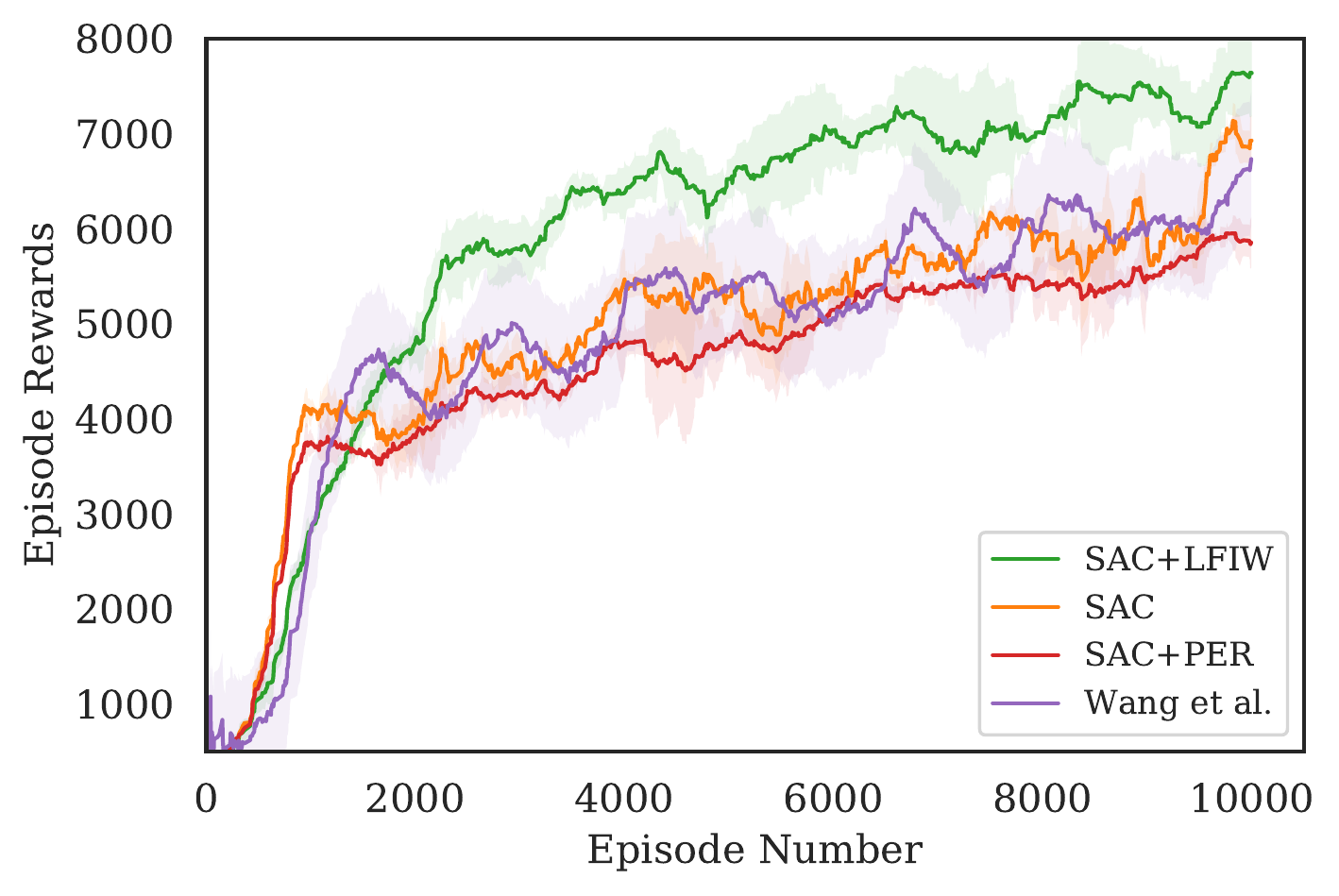}
\caption{Humanoid-v2}
\end{subfigure}
    \caption{Learning curvers for the OpenAI gym continuous control tasks using SAC~\cite{haarnoja2018soft}. The shaded region represents the standard deviation of the average evaluation over 5 trials.}
    \label{fig:sac-results}
\end{figure*}

\begin{figure*}
\centering
\begin{subfigure}{0.23\textwidth}
\includegraphics[width=\textwidth]{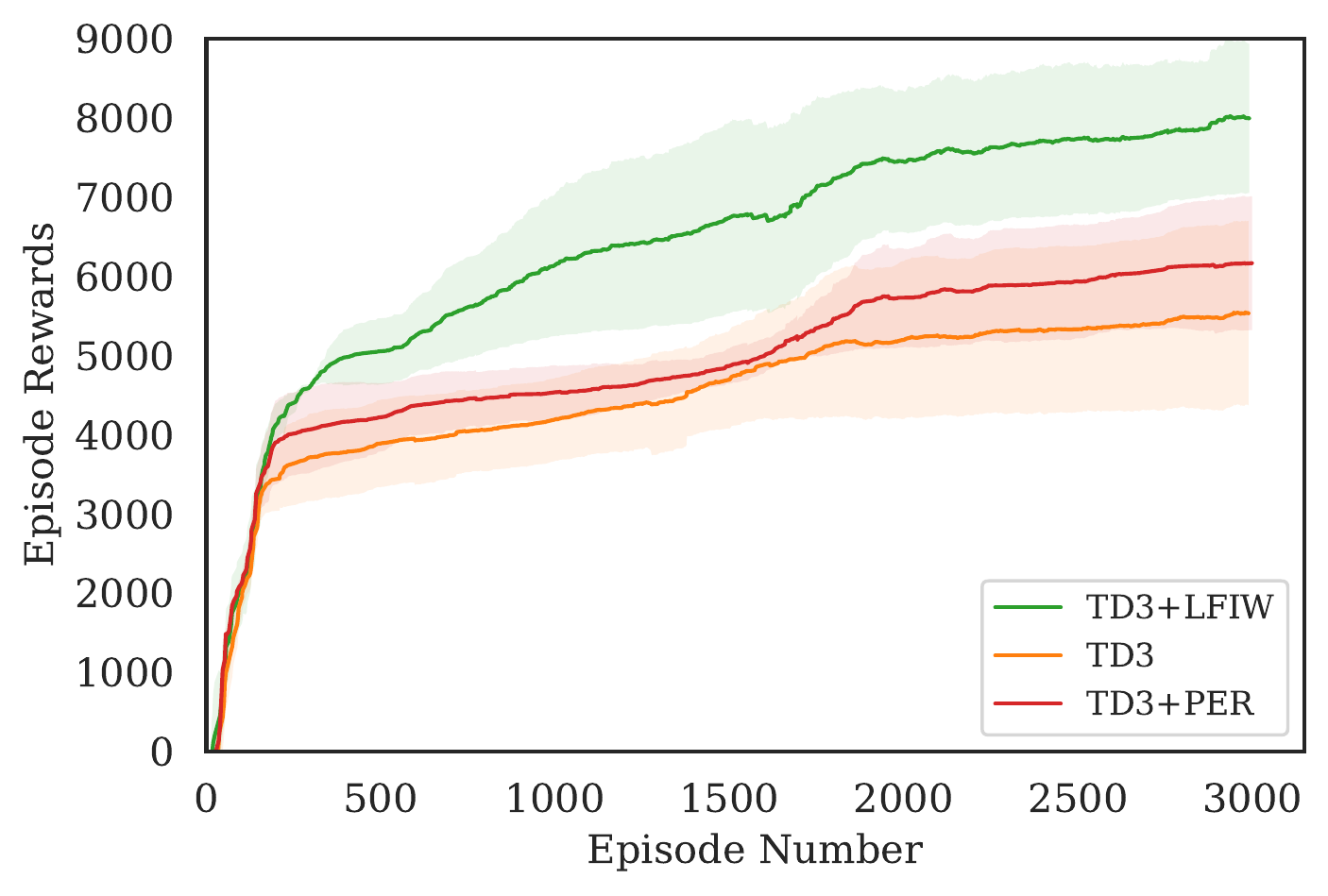}
\caption{HalfCheetah-v2}
\end{subfigure}
~
\begin{subfigure}{0.23\textwidth}
\includegraphics[width=\textwidth]{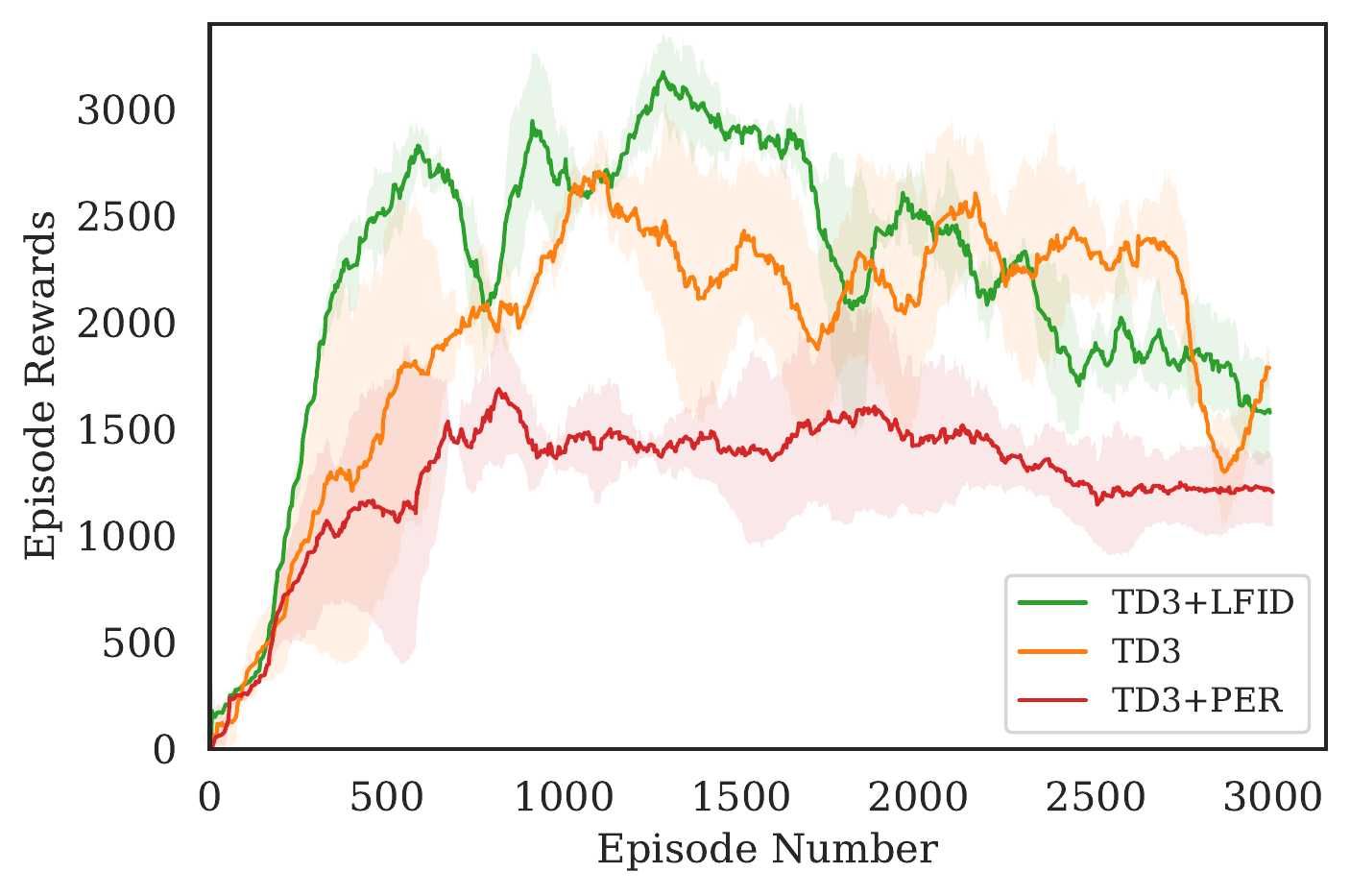}
\caption{Hopper-v2}
\end{subfigure}
~
\begin{subfigure}{0.23\textwidth}
\includegraphics[width=\textwidth]{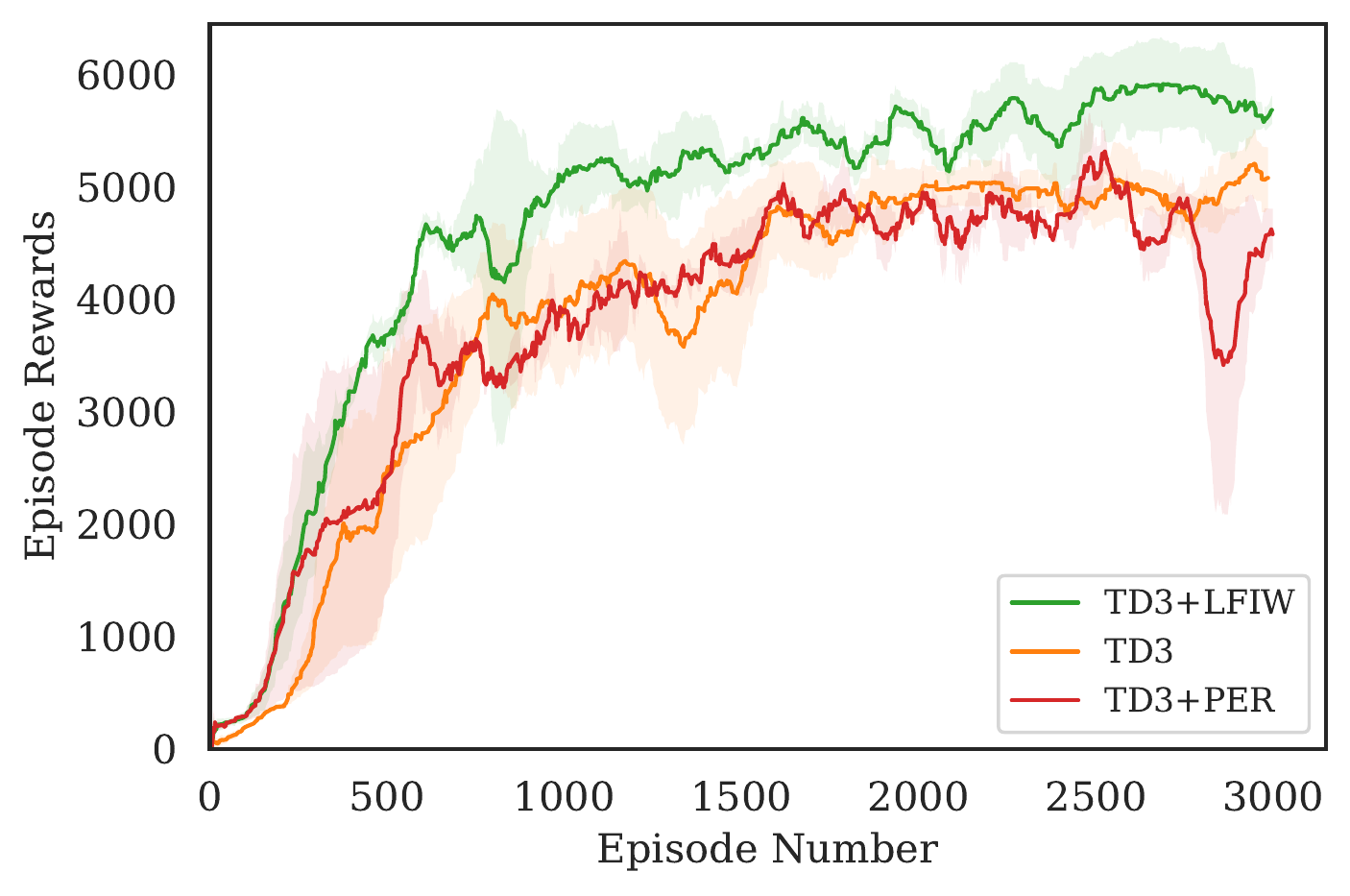}
\caption{Walker2d-v2}
\end{subfigure}
~
\begin{subfigure}{0.23\textwidth}
\includegraphics[width=\textwidth]{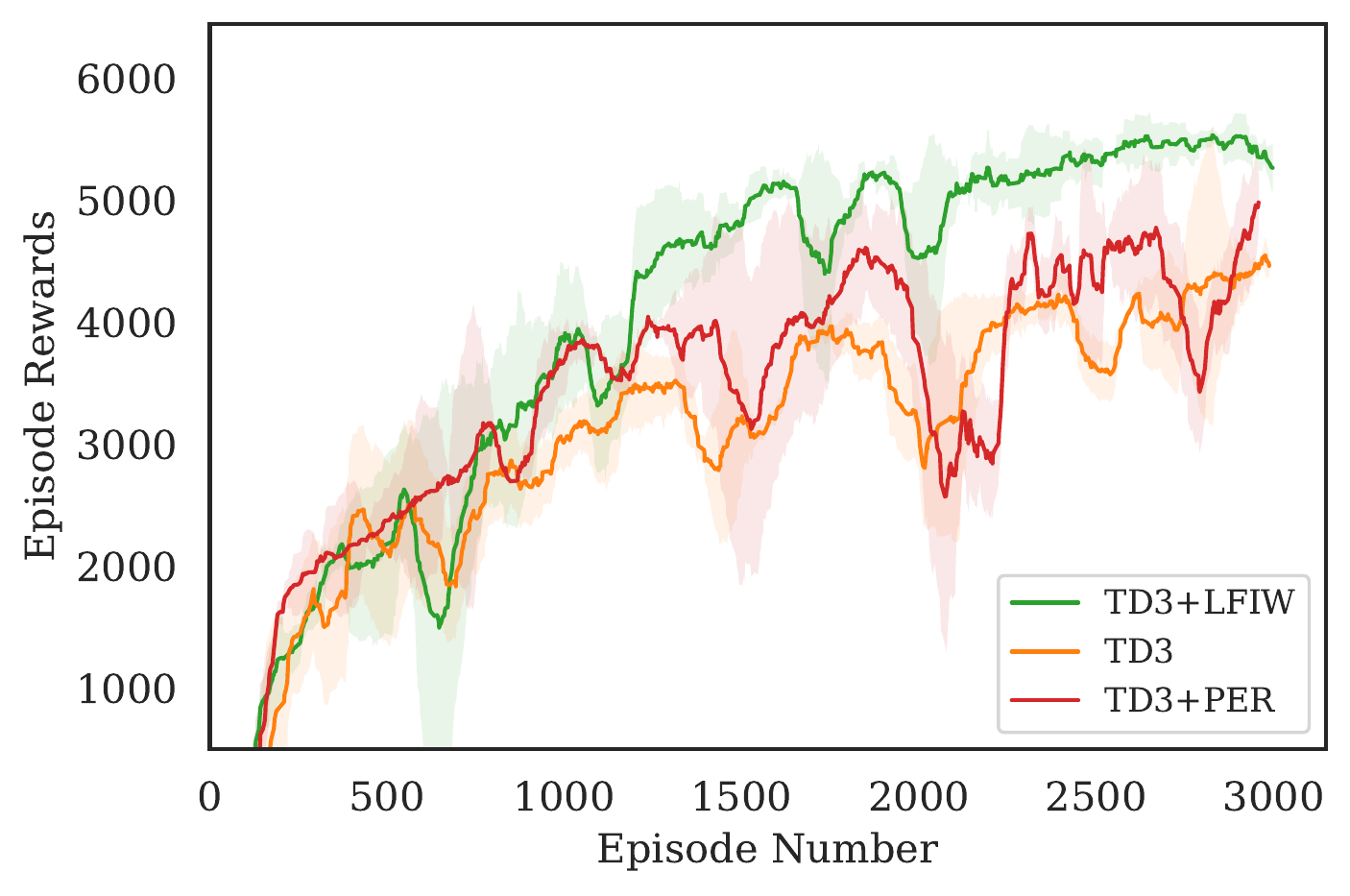}
\caption{Ant-v2}
\end{subfigure}
    \caption{Learning curvers for the OpenAI gym continuous control tasks using TD3~\cite{fujimoto2018addressing}. The shaded region represents the standard deviation of the average evaluation over 5 trials. We did not include {Humanoid} as the original TD3 algorithm fails to learn successfully.}
    \label{fig:td3-results}
\end{figure*}

We combine the proposed prioritization approach over two popular actor-critic algorithms, namely Soft-Actor Critic (SAC,~\cite{haarnoja2018soft}) and Twin Delayed Deep Deterministic policy gradient (TD3,~\cite{fujimoto2018addressing}). We compare our method with alternative approaches to prioritization; these include uniform sampling over the replay buffer (adopted by the original SAC and TD3 methods) and prioritization experience replay based on TD-error~\cite{schaul2015prioritized}. We evaluate over a range of Mujoco~\cite{todorov2012mujoco} continuous control tasks from the OpenAI gym benchmark~\cite{brockman2016openai}, including {HalfCheetah}, {Walker}, {Ant}, {Hopper} and {Humanoid} (-v2). Our implementations are all based on RLkit\footnote{\href{https://github.com/vitchyr/rlkit}{https://github.com/vitchyr/rlkit}} which gives a fair comparison between the baselines and our approach given the difficulty to reproduce the best RL results~\cite{pineau2020improving,engstrom2020implementation}. %

Our method introduces additional hyperparameters compared to the vanilla approaches, namely the temperature $T$, the size of the fast replay buffer $|\gD_{\mathrm{f}}|$ and the architecture for the density estimator $w_\psi$. To ensure fair comparisons against the baselines, we use the same hyperparameters as the original algorithms when it is available. %
For all environments we use the following default hyperparameters for likelihood-free importance weighting: $T = 5$, $|\gD_{\mathrm{f}}| = 10^4$, $|\gD_{\mathrm{s}}| = 10^6$, $w_\psi$ is parameterized using a two-layer fully connected network with $256$ neurons in each layer and ReLU activations. %
We use the divergence under $f(u) = u \log u - (1 + u) \log (1 - u)$ (\textit{i.e}, Jensen Shannon divergence) for better numerical stability. We perform likelihood-free importance weighting after 100 episodes to allow a decent amount of slow experiences to be collected. We include more experimental details in Appendix~\ref{app:exp}.

\subsection{Evaluation}

\begin{table}[htbp]
\centering
\caption{Max-performance attained by a given environment timestep for the Mujoco control
tasks. We report the mean maximum attained performance over 5 random seeds and standrad deviation.
}
\begin{tabular}{lccccc}
\toprule
Env               & Hopper-v2 & Walker-v2 & Cheetah-v2 & Ant-v2  & Humanoid-v2 \\ \midrule
Timesteps         & 1M     & 1M     & 1M      & 1M   & 5M      \\ \midrule
SAC~\cite{haarnoja2018soft}               &  2004 $\pm$ 356 &  \textbf{3862} $\pm$ 106 &   6548 $\pm$ 635 & 3138 $\pm$ 283 &  5515 $\pm$ 329  \\
SAC + PER~\cite{schaul2015prioritized}         &  1853 $\pm$ 106 &  3210 $\pm$ 418 &  6816 $\pm$ 531  & 2853 $\pm$ 132 &   4650 $\pm$ 315  \\
SAC + ERE~\cite{wang2019boosting} &  1759 $\pm$ 234 &  3601 $\pm$ 485 &  6666 $\pm$ 589  & 3346 $\pm$ 116 &   5586 $\pm$ 705  \\
SAC + LFIW        &  \textbf{2395} $\pm$ 212 &  \textbf{3855} $\pm$ 224 & \textbf{7037} $\pm$ 629 & \textbf{3857} $\pm$ 221 &   \textbf{6436} $\pm$ 254   \\
\midrule \midrule
TD3~\cite{fujimoto2018addressing}               &  2486 $\pm$ 125 &  4212 $\pm$ 456  &  4295 $\pm$ 523  & 2969 $\pm$ 202 & -    \\
TD3 + PER~\cite{schaul2015prioritized}         &  1704 $\pm$ 228 & 4268 $\pm$ 278 &   4766 $\pm$ 325  & 3679 $\pm$ 156 & -    \\
TD3 + LFIW        &  \textbf{3003} $\pm$ 261 & \textbf{5159} $\pm$ 189  &  \textbf{6184} $\pm$ 876   & \textbf{3864} $\pm$ 205 & -    \\
\bottomrule
\end{tabular}

\end{table}

We use (+LFIW) to denote our likelihood-free importance weighting method and (+PER) to denote prioritization with TD error~\cite{schaul2015prioritized}~\footnote{We use $\alpha = 0.6, \beta = 0.4$ in PER.}.
Along with PER, we also use Emphasizing Recent Experience (ERE,~\cite{wang2019boosting}) as a baseline experience replay method for SAC.
Figure~\ref{fig:sac-results} and Figure~\ref{fig:td3-results} shows the total average return of the evaluation episodes during training for methods based on SAC and TD3 respectively. We use five different random seeds for each algorithm and environment combination, where each evaluation is performed every 1000 timesteps. The solid curves correspond to the mean and the shaded region correspond to one standard deviation. Humanoid on TD3 is not reported as the original TD3 method fails to learn the policy properly.

The empirical results demonstrates that in terms of sample complexity and final performance, our LFIW method is able to outperform baseline methods on most tasks and has similar performance on the remaining tasks; we achieve a noticeable improvement on {HalfCheetah}, {Ant} and {Humanoid} with SAC, and on {HalfCheetah}, {Walker} and {Ant} with TD3. In these cases, LFIW is able to achieve faster improvements compared to the baselines; we attribute this to better $Q$-function estimation for the current policy, when the replay buffer contains experiences with worse performance. 

In comparison, PER and ERE do not perform very favorably against uniform sampling; a similar phenomenon has been observed by~\cite{novati2018remember} for PER on other actor-critic algorithms, such as DDPG~\cite{lillicrap2015continuous} and PPO~\cite{schulman2017proximal}. We also considered combining PER with LFIW, but achieved little initial success. We believe this is the case because PER is designed for $Q$-learning instead of actor-critic methods, where learning the max $Q$-function is the objective. 

In other cases, our LFIW method achieves similar performance compared to the baselines. One interesting case is the Hopper environment for SAC and TD3, where the initial performance of LFIW rises faster than the baselines but then plateaus at a similar level as the baselines (e.g. starting from 1200 episodes). We notice that in this case, the SAC performance reaches near-optimal at around 200 episodes; the ``slow'' replay buffer has a size of $10^6$, so the replay buffer at 1200 episodes contains almost entirely experiences from the near-optimal policy. Performing LFIW brings little improvement in this case as the slow replay buffer already contains experiences that are near on-policy.

\subsection{Additional analyses}
To illustrate the advantage of our method, we perform further analyses over the classification accuracy of $w_\psi$ and the quality of the $Q$-function estimates over the Humanoid-v2 environment trained with SAC and SAC + LFIW. In Appendix~\ref{app:exp}, we also include additional ablation studies over the hyperparameters introduced by LFIW, including temperature $T$, replay buffer size $|\gD_f|$ and number of hidden units in $w_\psi$. We observe that SAC + LFIW is insensitive to these hyperparameter changes.

\paragraph{Accuracy of $w_\psi$} We use $w_\psi$ to discriminate two types of experiences; experiences sampled from the policy trained with SAC for 5M steps are labeled positive, and the mixture of experiences sampled from policies trained for 1M to 4M steps are labeled negative. With the $w_\psi$ predictions, we obtain a precision of $87.3\%$ and an accuracy of $73.1\%$. This suggests that the importance weights tends to be higher for on-policy data as desired, and the weights indeed allows the replay buffer to be closer to on-policy experiences.

\paragraph{Quality of $Q$-estimates} We compare the quality of the $Q$-estimates between SAC and SAC+LFIW, where we sample 20 trajectories from each policy, and obtain the ``ground truth'' via Monte Carlo estimates of the true $Q$-value. We then evaluate the learned $Q$-function estimates and compare their correlations with the ground truth values. For the SAC case, the Pearson and Spearman correlations are $0.41$ and $0.11$ respectively, whereas for SAC+LFIW method they are $0.74$ and $0.42$ (higher is better). 
This shows how our $Q$-function estimates are much more reflective of the ``true'' values, which explains the improvements in sample complexity and the performance of the learned policy.
\section{Conclusion}
In this paper, we propose a principled approach to prioritized experience replay for TD-learning of $Q$-value functions in actor-critic methods, where we re-weigh the replay buffer to be closer to on-policy experiences.
To implement this in practice, we assign weights to the replay buffer based on their estimated density ratios against the stationary distribution. 
These density ratios are estimated via samples from fast and slow replay buffers, which reflect on-policy and off-policy experiences respectively. 
Our methods can be readily applied to deep reinforcement learning methods based on actor-critic approaches. Empirical results on SAC and TD3 demonstrate that our method based on likelihood-free importance weighting is able to achieve superior sample complexity on most of the challenging Mujoco environments compared to other methods. %

In future work, 
we are interested in extending our work to larger-scale discrete environments such as Atari~\cite{bellemare2016unifying} and investigate proper representations for better estimation of the density ratios. It would also be interesting to apply a prioritization scheme with a learned dynamics model of the environment, which could facilitate better planning in model-based reinforcement learning.

\section*{Acknowledgements}
This research was supported by AFOSR (FA9550-19-1-0024), NSF (\#1651565, \#1522054, \#1733686), ONR, and FLI.
We would like to acknowledge Nvidia for donating DGX-1, and Vector Institute 
for providing resources for this research. 

\bibliographystyle{plainnat}
\bibliography{bib}

\onecolumn
\appendix

\section{Algorithm}
\label{app:alg}
\begin{algorithm}[h]
\caption{Actor Critic with Likelihood-free Importance Weighted Experience Replay}
\label{alg:ac-lfiw}
\begin{algorithmic}[1]
\REPEAT 
\FOR{each environment step} 
\STATE gather new transition tuples $(s, a, r, s')$
\STATE update $(s, a, s, s')$ to $\gD_{\mathrm{s}}$ (slow replay buffer) and $\gD_{\mathrm{f}}$ (fast replay buffer)
\ENDFOR
\STATE remove stale experiences in $\gD_{\mathrm{s}}, \gD_{\mathrm{f}}$ ($|\gD_{\mathrm{f}}| < |\gD_{\mathrm{s}}|$)

\IF{$|\gD_{\mathrm{s}}|$ exceeds some threshold}
\STATE obtain samples from $\gD_{\mathrm{s}}$ and $\gD_{\mathrm{f}}$
\STATE update $w_\psi$ with loss function $L_w(\psi)$ (Eq.~\ref{eq:w-train-obj})
\STATE assign $\tilde{w}_\psi$ according to Eq.~\ref{eq:w-tilde}
\ELSE
\STATE $\tilde{w}_\psi = 1$ (no re-weighting)
\ENDIF
\STATE obtain estimates for $B^\pi Q_\theta$ with base algorithm
\STATE update $Q_\theta$ with loss function $L_Q(\theta; \gD_{\mathrm{s}}, \tilde{w})$ (Eq.~\ref{eq:final-q-obj})
\STATE update $\pi_\phi$ and value network (if available) with base algorithm
\UNTIL{Stopping criterion}
\RETURN{$Q_\theta, \pi_\phi$}

\end{algorithmic}
\end{algorithm}

\section{Proofs}
\label{app:proof}
\contraction*

\begin{proof}
From the definitions of $\norm{\cdot}_d$ and $\gB^\pi$, we have:
\begin{align}
    & \norm{\gB^\pi Q - \gB^\pi Q'}^2_{d} \\ %
    = \ & \bb{E}_{(s, a) \sim d}[(\gamma \bb{E}_{s', a'}[Q(s', a')] - \gamma \bb{E}_{s', a' }[Q'(s', a')])^2] \nonumber \\
    = \ & \gamma^2  \bb{E}_{(s, a) \sim d}[(\bb{E}_{s', a'}[Q(s', a') - Q'(s', a')])^2] \nonumber \\
    \leq  \ & \gamma^2 \bb{E}_{(s, a) \sim d}[\bb{E}_{s', a'}[(Q(s', a') - Q'(s', a'))^2]] \label{eq:contraction-proof-jensen} \\
    = \ & \gamma^2 \bb{E}_{(s, a) \sim d'}[(Q(s, a) - Q'(s, a))^2] \label{eq:contraction-proof-stationary} \\
    = \ & \gamma^2 \norm{Q - Q'}^2_{d'} 
\end{align}
where $s' \sim P(\cdot | s, a), a' \sim \pi(\cdot | s')$ and $$d'(s', a') = \sum_{s, a} P(s' | s, a) \pi(a' | s') d(s, a)$$ represents the state-action distribution of the next step when the current distribution is $d$. We use Jensen's inequality over the convex function $(\cdot)^2$ in Eq.~\ref{eq:contraction-proof-jensen}. 
Since $d^\pi$ is the stationary distribution, $d = d' \iff d = d^\pi, a.e.$, so the if direction holds. 

For the ``only if'' case, we construct a counter-example for all $d \neq d'$. Without loss of generality, assume $\forall s, a \in \gS \times \gA, Q'(\vx) = 0$. The following functionals of $Q$
\begin{align*}
    h(Q) & := \norm{\gB^\pi Q - \gB^\pi Q'}^2_{d} / \gamma^2 \\ & =  \bb{E}_{(s, a) \sim d}[(\bb{E}_{s', a'}[Q(s', a')])^2] \\
    g(Q) & := \norm{Q - Q'}^2_{d} \\ & = \bb{E}_{(s, a) \sim d}[Q(s, a)^2]
\end{align*}
corresponds to the quantities at the two ends of the contraction argument. Our goal is to find some $Q \in \gQ$ such that $h(Q) - g(Q) > 0$, which would complete the contradiction. 
We can evaluate the functional derivatives for $h(Q)$ and $g(Q)$:
\begin{gather*}
    \frac{\diff h}{\diff Q}(s', a') = 2 \sum_{s, a} d(s, a) \gE(s, a) P(s' | s, a) \pi(a' | s') \\
    \frac{\diff g}{\diff Q}(s, a) = 2 d(s, a) Q(s, a)
\end{gather*}
where $\gE(s, a) = \bb{E}_{s'' \sim P(\cdot | s, a), a'' \sim \pi(\cdot | s'')}[Q(s'', a'')]$ is the expected $Q$ function of the next step when the current step is at $(s, a)$. Now let us consider some $Q_0$ such that for some constant $q > 0$, $\forall s, a \in \gS \times \gA, Q_0(s, a) = q$ . Let us then evaluate both functional derivatives at $Q_0$:
\begin{align*}
    \left.\frac{\diff h}{\diff Q}(s', a')\right\vert_{Q_0} = & \ 2 \sum_{s, a} d(s, a) q P(s' | s, a) \pi(a' | s') \\
    = & \ 2 d'(s', a') q \\
    \left.\frac{\diff g}{\diff Q}(s, a)\right\vert_{Q_0} = & \ 2 d(s, a) q
\end{align*}
where $\gE(s, a) = q$ under the current $Q_0$. Because $d'$ and $d$ are not equal almost everywhere (from the assumption of $d$ not being the stationary distribution), there must exist some non-empty open set $\Gamma \in \gS \times \gA$ where $\int_\Gamma (d'(s, a) - d(s, a)) \diff s \diff a > 0$. We can then add a function $\epsilon: \gS \times \gA \to \R$ such that $\epsilon(s, a) = \nu \cdot \mathbb{I}((s, a) \in \Gamma)$ where $\mathbb{I}$ is the indicator function and $\nu$ is an infinitesimal amount. Now let us evaluate $(h - g)$ at $(Q_0 + \epsilon)$:
\begin{align*}
     & (h - g)(Q_0 + \epsilon) \\ 
    = \ &  (h - g)Q_0 + \left(\frac{\diff h}{\diff Q} - \frac{\diff g}{\diff Q} \right) \epsilon + o(\nu) \nonumber \\
    = \ & (q^2 - q^2) + 2q\int_\Gamma (d'(s, a) - d(s, a)) \nu \diff s \diff a + o(\nu) > 0
\end{align*}
Therefore, the proposed function $(Q + \epsilon)$ is the contradiction we need.
\end{proof}

\section{Additional Experimental Details}
\label{app:exp}

\subsection{Setup on Chain MDP}
The chain MDP considered has deterministic transitions, so we make the policy stochastic to make sure that a stationary distribution exists. In each epoch over all the state-action pairs, we use the following TD learning update over tabular data to simulate the effect of weighting with fixed learning rate $\eta$:
\begin{align}
   Q(s, a) \rightarrow Q(s, a) + (1 - (1 -\eta)^{w(s, a)}) (\gB^\pi Q(s, a) - Q(s, a))
\end{align}
where $w(s, a)$ is the weight (uniform, TD error or $d^\pi$) which simulates the number of TD updates with learning rate $\eta$. The weights are normalized to have a mean value of $1$ which makes number of updates per epoch the same across different methods.

\begin{table}[h]
    \centering
    \caption{Additional hyperparameters for SAC~\cite{haarnoja2018soft}}
    \begin{tabular}{c|c}
    \toprule
        Parameter & Value \\\midrule
        optimizer & Adam~\cite{kingma2014adam} \\
        learning rate & $3 \times 10^{-4}$\\
        discount & $0.99$ \\
        number of samples per minibatch & 256 \\
        nonlinearity & ReLU\\
        target smoothing coefficient & $5 \times 10^{-3}$ \\
         \bottomrule
    \end{tabular}
    
    \label{tab:hyper-sac}
\end{table}

\begin{table}[h]
    \centering
    \caption{Additional hyperparameters for TD3~\cite{fujimoto2018addressing}}
    \begin{tabular}{c|c}
    \toprule
        Parameter & Value \\\midrule
        optimizer & Adam~\cite{kingma2014adam}\\
        learning rate & $10^{-3}$\\
        discount & $0.99$ \\
        number of samples per minibatch & 256 \\
        nonlinearity & ReLU \\
        exploration policy & $\gN(0, 1)$ \\
         \bottomrule
    \end{tabular}
    
    \label{tab:hyper-td3}
\end{table}

\subsection{Ablation studies}

To demonstrate the stability of our method across different hyperparameters, we conduct further analyses over the key hyperparameters, %
including the temperature $T$ in Eq.~\ref{eq:w-tilde}, the size of the fast replay buffer $|\gD_{\mathrm{f}}|$, and the number of hidden units in the classifier model $w_\psi$. We consider running the SAC+LFIW method on the Walker-v2 environment with 1000 episodes using all the default hyperparameters unless explicitly changed. 

\paragraph{Temperature $T$} The temperature $T$ affects the variances of the weights assigned; a larger $T$ makes the weights more similar to each other, while a smaller $T$ relies more on the outputs of the classifier. Since we are using finite replay buffers, using a larger temperature reduces the chances of negatively impacting performance due to $w_\psi$ overfitting the data. We consider $T = 1, 2.5, 5, 7.5, 10$ in Figure~\ref{fig:ablation:t}; all cases have similar sample efficiencies except for $T = 1$.
Similarly, we also perform a similar analysis on Humanoid-v2 with SAC in Figure~\ref{fig:humanoid_temperature}.
We observe a similar dependency on $T$ as in Walker where the sample efficiency with $T = 1$ 
is significantly worse that for the other hyperparameters considered, which shows that 
overfitting the data can easily be avoided by using a higher temperature value even in 
higher-dimensional state-action distributions.

\paragraph{Replay buffer sizes $|\gD_{\mathrm{f}}|$} The replay buffer sizes $|\gD_{\mathrm{f}}|$ affects the amount of experiences we treat as ``on-policy''. Larger $|\gD_{\mathrm{f}}|$ reduces the risk of overfitting while increasing the chances of including more off-policy data. We consider $|\gD_{\mathrm{f}}| = 1000, 10000, 50000, 100000$, corresponding to $1$ to $100$ episodes. We note that $|\gD_{\mathrm{s}}| = 10^6$, so even for the largest $\gD_{\mathrm{f}}$, $\gD_{\mathrm{s}}$ is significantly larger. The performance are relatively stable despite a small drop for $|\gD_{\mathrm{f}}| = 100000$.

\paragraph{Hidden units of $w_\psi$} The number of hidden units at each layer affects the expressiveness of the neural network. While networks with more hidden units are more expressive, they are easier to overfit to the replay buffers. We consider hidden layers with $128, 256$ and $512$ neurons respectively. While the smaller network with $128$ units is able to achieve superior performance initially, the other configurations are able to catch up at around 1000 episodes.

\newpage

\begin{figure*}
    \centering
\begin{subfigure}{0.48\textwidth}
\includegraphics[width=\textwidth]{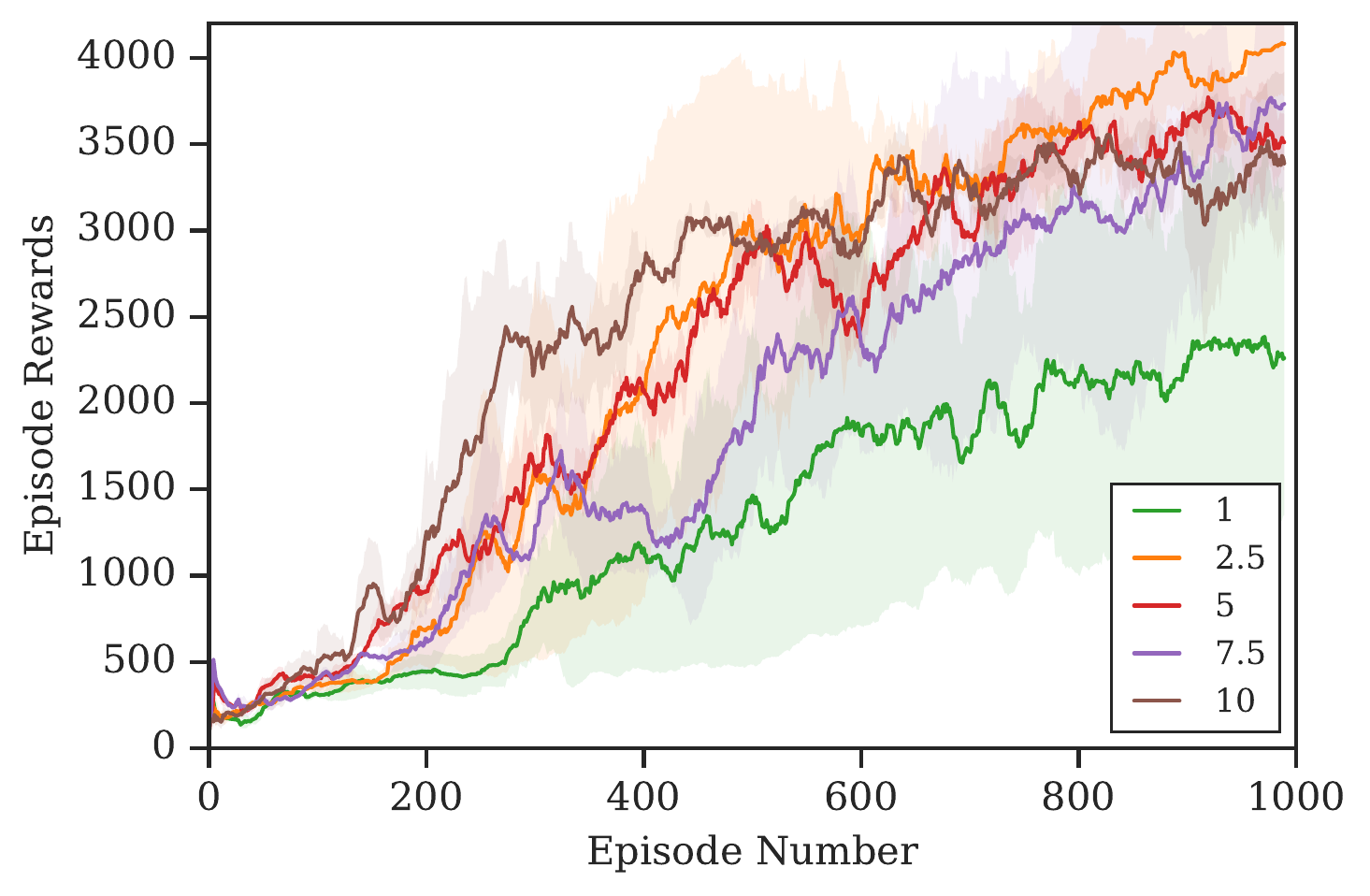}
\caption{Temperature $T$}
\label{fig:ablation:t}
\end{subfigure}
~
\begin{subfigure}{0.48\textwidth}
\includegraphics[width=\textwidth]{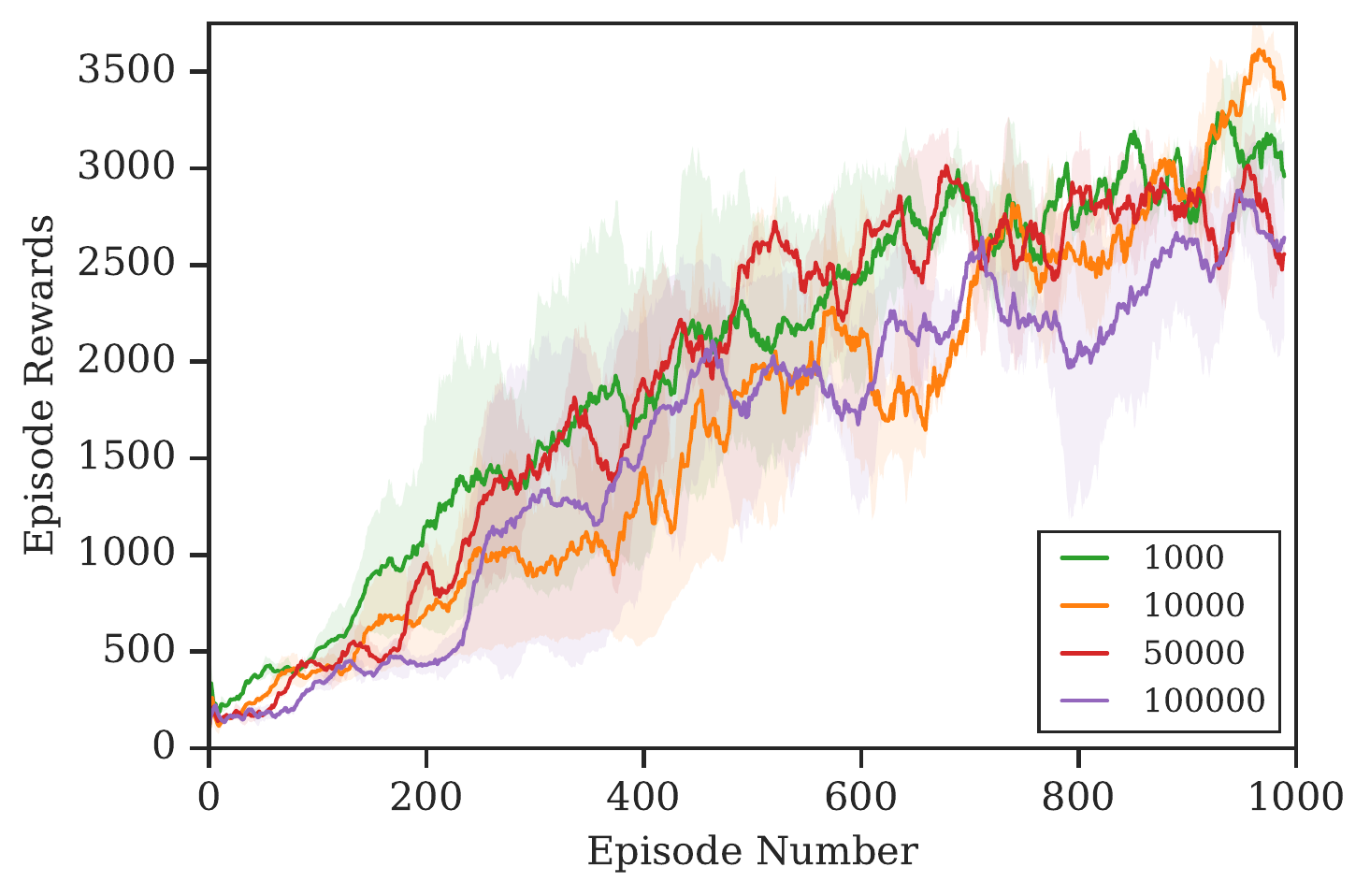}
\caption{Fast replay buffer size $|\gD_{\mathrm{f}}|$}
\label{fig:ablation:f}
\end{subfigure}
~
\begin{subfigure}{0.5\textwidth}
\includegraphics[width=\textwidth]{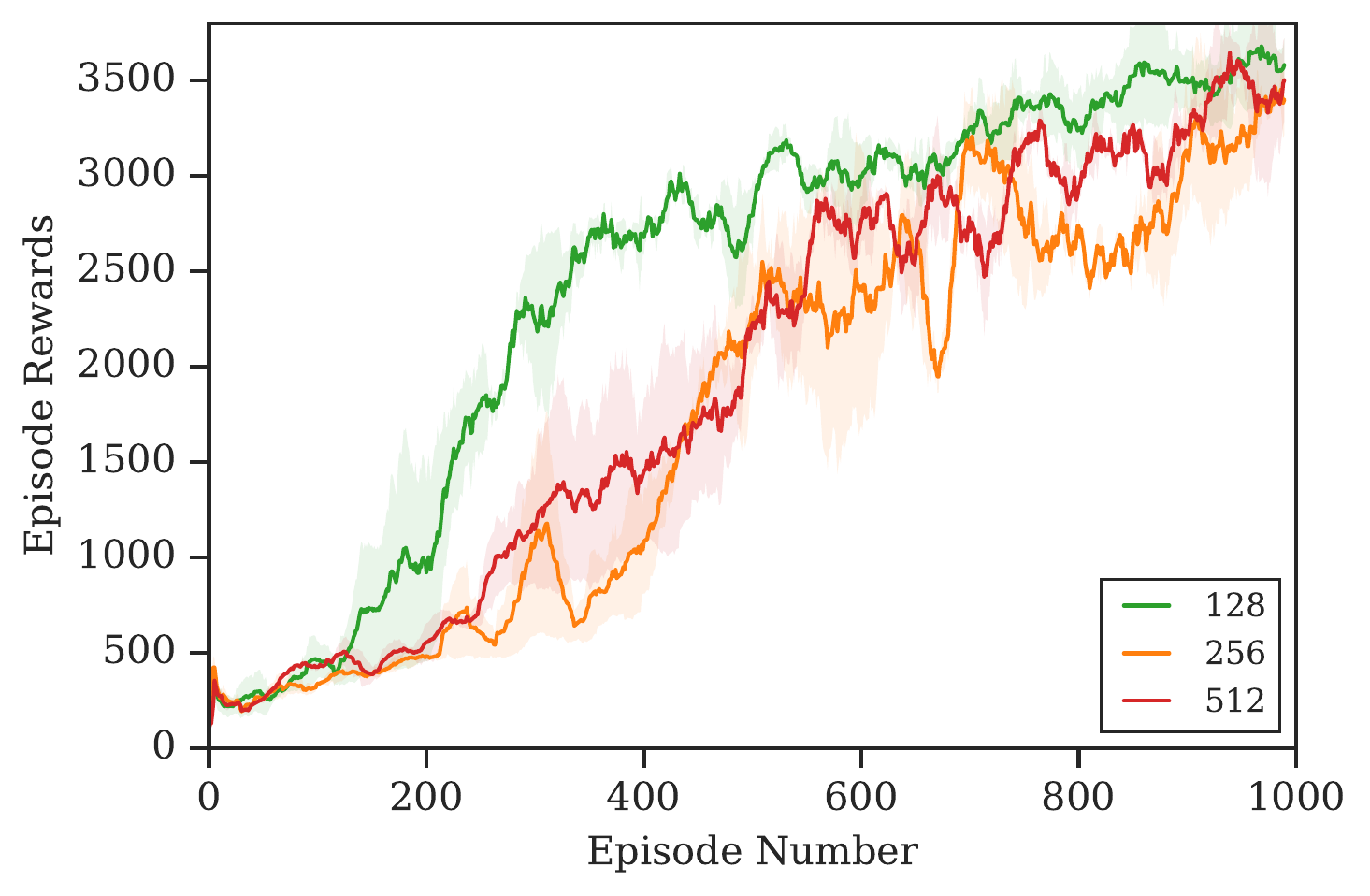}
\caption{Hidden layer size of $w_\psi$}
\label{fig:ablation:h}
\end{subfigure}
~
\caption{Hyperparameter sensitivity analyses on Walker2d-v2 with SAC.}
    \label{fig:ablation}
\end{figure*}

\begin{figure}[]
    \centering
    \includegraphics[width=0.6\textwidth]{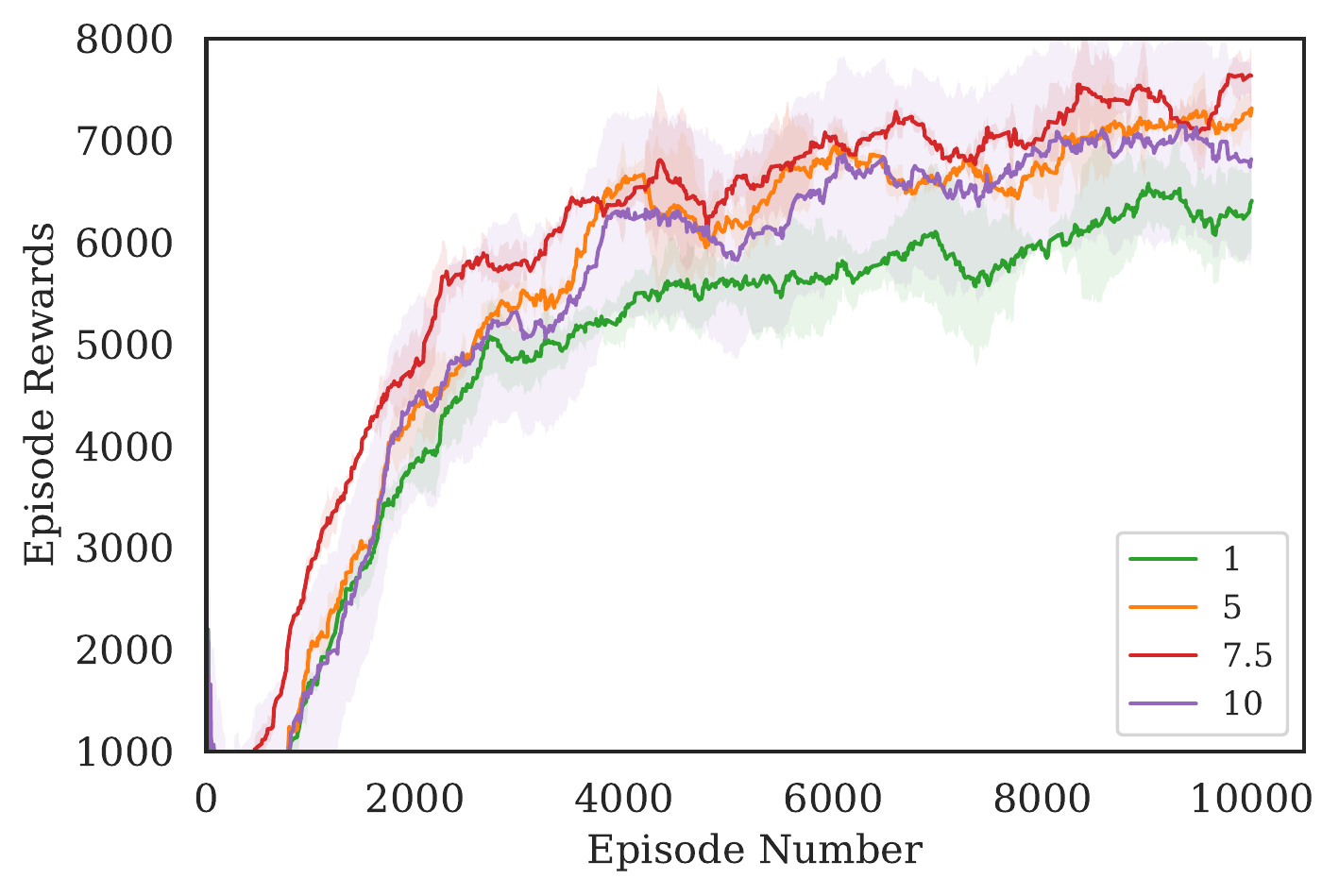}
    \caption{Temperature sensitivity on Humanoid-v2 with SAC}
    \label{fig:humanoid_temperature}
\end{figure}

\end{document}